%% file: bandits_main.tex
\documentclass[reqno]{amsart}

\input{bandits_preamble} 
\begin{document}

\input{bandits_info} 
\input{bandits_abstract}
\maketitle
\input{bandits_introduction}
\input{bandits_prob_setup}

\input{bandits_k1_analysis}
\input{bandits_gen_k_analysis}
\input{bandits_conclusion}
\bibliographystyle{unsrt}
\bibliography{bandits_references}
\input{bandits_appendix.tex}

\end{document}

%% file: bandits_preamble.tex
\usepackage[margin=1in]{geometry}
\usepackage{amsmath}%
\usepackage{amsfonts}%
\usepackage{amssymb}%
\usepackage{graphicx}
\usepackage{subfig}
\usepackage{color}
\usepackage{enumitem}
\usepackage{algpseudocode}
\usepackage{algorithmicx}
\usepackage{algorithm}
\usepackage{bbm}
\newtheorem{theorem}{Theorem}
\theoremstyle{plain}

\newtheorem{assumption}{Assumption}

\newtheorem{claim}{Claim}

\newtheorem{definition}{Definition}

\newtheorem{lemma}{Lemma}

\newtheorem{remark}{Remark}

\numberwithin{equation}{section}

\def\url#1{\expandafter\string\csname #1\endcasname}

\newcommand{\calA}{\ensuremath{\mathcal{A}}}
\newcommand{\calB}{\ensuremath{\mathcal{B}}}
\newcommand{\calC}{\ensuremath{\mathcal{C}}}

\newcommand{\calH}{\ensuremath{\mathcal{H}}}
\newcommand{\calF}{\ensuremath{\mathcal{F}}}

\newcommand{\calP}{\ensuremath{\mathcal{P}}}

\newcommand{\calT}{\ensuremath{\mathcal{T}}}

\newcommand{\algo}{\texttt{ALG}}
\newcommand{\adver}{\texttt{ADV}}

\newcommand{\CABdk}{\text{CAB}(\ensuremath{d,k})}

\newcommand{\matR}{\ensuremath{\mathbb{R}}}
\newcommand{\matA}{\ensuremath{\mathbf{A}}}
\newcommand{\expec}{\ensuremath{\mathbb{E}}}

\newcommand{\norm}[1]{\parallel{#1}\parallel}
\newcommand{\abs}[1]{|{#1}|}
\newcommand{\set}[1]{\left\{{#1}\right\}}

\newcommand{\vecx}{\mathbf x}
\newcommand{\vecp}{\mathbf p}
\newcommand{\veci}{\mathbf i}
\newcommand{\vecu}{\mathbf u}

\newcommand{\vecy}{\mathbf y}
\newcommand{\vecz}{\mathbf z}

%% file: bandits_info.tex
\title[Continuum armed bandit problem of few variables in high
dimensions]{Continuum armed bandit problem of few variables in high dimensions}
\author{Hemant Tyagi and Bernd G\"artner}
\thanks{Hemant Tyagi and Bernd G\"artner are with the Institute of Theoretical
Computer Science, ETH Z\"urich (ETHZ), CH-8092 Z\"urich, Switzerland. Emails:
htyagi@inf.ethz.ch, gaertner@inf.ethz.ch}
\thanks{A preliminary version of this work appeared in the proceedings of the
$11^{th}$ Workshop on Approximation and Online Algorithms (WAOA), 2013. The final
publication is available at: http://dx.doi.org/10.1007/978-3-319-08001-7\_10}
\keywords{Bandit problems, continuum armed bandits, functions of few variables,
online optimization.}

%% file: bandits_abstract.tex
\begin{abstract}
We consider the stochastic and adversarial settings of continuum armed bandits
where the arms are indexed by $[0,1]^d$. The reward functions $r:[0,1]^d
\rightarrow \matR$ are assumed to \textit{intrinsically} depend on at most $k$
coordinate variables implying $r(x_1,\dots,x_d) = g(x_{i_1},\dots,x_{i_k})$ for
distinct and unknown $i_1,\dots,i_k \in \set{1,\dots,d}$ and some locally
H\"{o}lder continuous $g:[0,1]^k \rightarrow \matR$ with exponent $\alpha \in
(0,1]$. Firstly, assuming $(i_1,\dots,i_k)$ to be fixed across time, we propose
a simple modification of the CAB1 algorithm where we construct the discrete set
of sampling points to obtain a bound of $O(n^{\frac{\alpha+k}{2\alpha+k}} (\log
n)^{\frac{\alpha}{2\alpha+k}} C(k,d))$ on the regret, with $C(k,d)$ depending at
most polynomially in $k$ and sub-logarithmically in $d$. The construction is
based on creating partitions of $\set{1,\dots,d}$ into $k$ disjoint subsets and
is probabilistic, hence our result holds with high probability. Secondly we
extend our results to also handle the more general case where $(i_1,\dots,i_k)$
can change over time and derive regret bounds for the same.
\end{abstract}

%% file: bandits_introduction.tex
%
\section{Introduction} \label{sec:intro}
In online decision making problems, a player is required to play a strategy,
chosen from a given set of strategies $S$, over a period of $n$ trials or
rounds. Each strategy has a reward associated with it specified by a reward
function $r:S \rightarrow \matR$ which typically changes across time in a manner
unknown to the player. The aim of the player is to choose the strategies in a
manner so that the total expected reward of the chosen strategies is maximized. 
The performance of algorithms in online decision problems is measured in terms
of their \textit{regret} defined as the difference between the total expected
reward of the best constant (i.e. not varying with time) strategy and the
expected reward of the sequence of strategies played by the player. If the
regret after $n$ rounds is sub-linear in n, this implies as $n \rightarrow
\infty$ that the per-round expected reward of the player asymptotically
approaches that of the best fixed strategy. There are many applications of
online decision making problems such as routing \cite{Awerbuch04,Bansal03},
wireless networks \cite{Monteleoni03}, online auction mechanisms
\cite{Blum03,Kleinberg03}, statistics (sequential design of experiments
\cite{Lai85}) and economics (pricing \cite{Rothschild74}) to name a few. Broadly
speaking, there are two main types of online decision making problems depending
on the type of feedback the player receives at the end of each round.

\begin{enumerate}
\item \textit{Best expert problem.} In this problem the entire reward function
is revealed to the algorithm, as feedback at the end of each round.

\item \textit{Multi-armed bandit problem.} In this problem the algorithm only
receives the reward associated with the strategy that was played in the round.
\end{enumerate}

Multi-armed bandit problems have been studied extensively when the strategy set
$S$ is finite and optimal regret bounds are known within a constant factor 
\cite{Auer95,Auer02,Lai85}. On the other hand, the setting in which $S$ consists
of infinitely many arms has been an area of recent attention due to its
practical significance. Such problems are referred to as continuum armed bandit
problems and are the focus of this paper. Usually $S$ is considered to be a
compact subset of a metric space such as $\matR^d$. Some applications of these
problems are in: (i) online auction mechanism design \cite{Blum03,Kleinberg03}
where the set of feasible prices is representable as an interval and, (ii)
online oblivious routing \cite{Bansal03} where $S$ is a flow polytope. 

For a $d$-dimensional strategy space, if the only assumption made on the reward
functions is on their degree of smoothness then any multi-armed bandit algorithm
will incur worst-case regret which depends exponentially on $d$. To see this,
let $S=[-1,1]^d$ and say the reward function does not vary with time and is zero
in all but one orthant $\mathcal{O}$ of $S$ where it takes a positive value
somewhere. Clearly any algorithm would need in expectation $\Omega(2^d)$ trials
before it can sample a point from $\mathcal{O}$ thus implying a bound of
$\Omega(2^d)$ on regret incurred. 
More precisely, let $R(n)$ denote the cumulative
regret incurred by the algorithm after $n$ rounds. Bubeck et al. \cite{Bubeck2011ALT} showed
that $R(n) = \Omega(n^{\frac{d+1}{d+2}})$ after $n = \Omega(2^{d})$ plays for stochastic continuum armed
bandits\footnote{rewards sampled at each round in an i.i.d manner from an unknown probability distribution} 
with $d$-variate Lipschitz continuous mean reward functions defined
over $[0,1]^d$. Clearly the per-round expected regret $R(n)/n = \Omega(n^{\frac{-1}{d+2}})$ which
means that it converges to zero at a rate at least exponentially slow in $d$.
To circumvent this curse of dimensionality,
the reward functions are typically considered to be linear (see for example
\cite{McMahan04,Abernethy08}) or convex (see for example
\cite{Flaxman05,Kleinberg04}) for which the regret is polynomial in $d$ and
sub-linear in $n$. We consider the setting where the reward function $r:[0,1]^d
\rightarrow \matR$ depends on an unknown subset of $k$ \emph{active} coordinate
variables implying $r(x_1,\dots,x_d) = g(x_{i_1},\dots,x_{i_k})$. The
environment is allowed to sample the underlying function $g$ either in an i.i.d
manner from some fixed underlying distribution (stochastic) or arbitrarily
(adversarial).

\subsection*{Related Work} The continuum armed bandit problem was first
introduced in \cite{Agrawal95} for the case $d=1$ where an algorithm achieving a
regret bound of $o(n^{(2\alpha+1)/(3\alpha+1)+\eta})$ for any $\eta > 0$ was
proposed for local H\"{o}lder continuous\footnote{A function $r : S \rightarrow
\matR$ is H\"{o}lder continuous if $\abs{r(\vecx) - r(\vecy)} \leq L
\norm{\vecx-\vecy}^{\alpha}$ for constants $L >0$, $\alpha \in (0,1]$ and any
$\vecx,\vecy \in S$.} mean reward functions with exponent $\alpha \in (0,1]$. In
\cite{Kleinberg03} a lower bound of $\Omega(n^{1/2})$ was proven for this
problem. This was then improved upon in \cite{Kleinberg04} where the author
derived upper and lower bounds of $O(n^{\frac{\alpha+1}{2\alpha+1}}(\log
n)^{\frac{\alpha}{2\alpha+1}})$ and $\Omega(n^{\frac{\alpha+1}{2\alpha+1}})$
respectively. In \cite{Cope09} the author considered a class of mean reward
functions defined over a compact convex subset of $\matR^d$ which have (i) a
unique maximum $\vecx^{*}$, (ii) are three times continuously differentiable and
(iii) whose gradients are well behaved near $\vecx^{*}$. It was shown that a
modified version of the Kiefer-Wolfowitz algorithm achieves a regret bound of
$O(n^{1/2})$ which is also optimal. In \cite{Auer07improvedrates} the $d = 1$
case was treated, with the mean reward function assumed to only satisfy a local
H\"{o}lder condition around the maxima $\vecx^{*}$ with exponent $\alpha \in
(0,\infty)$. Under these assumptions the authors considered a modification of
Kleinberg's CAB1 algorithm \cite{Kleinberg04} and achieved a regret bound of
$O(n^{\frac{1+\alpha-\alpha\beta}{1+2\alpha-\alpha\beta}} (\log
n)^{\frac{\alpha}{1+2\alpha-\alpha\beta}})$ for some known $0 < \beta < 1$. In
\cite{Kleinberg08,Bubeck2011} the authors studied a very general setting for the
multi-armed bandit problem in which $S$ forms a metric space, with the reward
function assumed to satisfy a Lipschitz condition with respect to this metric.
In particular it was shown in \cite{Bubeck2011} that if $S = [0,1]^d$ and the
mean reward function satisfies a local H\"{o}lder condition with exponent
$\alpha \in (0,\infty)$ around any maxima $\vecx^{*}$ (the number of maxima
assumed to be finite) then their algorithm achieves \textit{rate} of growth of
regret which is\footnote{$u = \tilde{O}(v)$ means $u =
O(v)$ up to a logarithmic factor.} $\tilde{O}(\sqrt{n})$ and hence \textit{independent} of 
dimension\footnote{Note that there is still a factor that is exponential in $d$ in
the regret bound.} $d$.

Chen et al. \cite{Chen12} consider the problem of Bayesian
optimization of high dimensional functions by assuming the functions to
intrinsically depend on only a few relevant variables. 
They consider a stochastic environment but assume
the underlying reward functions to be samples from a Gaussian process (GP).
They propose a two-stage scheme where they first learn the set of active variables and
then apply a standard GP algorithm to perform Bayesian optimization over this 
identified set. In contrast, we consider both stochastic and adversarial environments
and assume H\"{o}lder continuous reward functions. Unlike \cite{Chen12}, we do not 
require to learn the set of active coordinate variables. We are also able to handle the
scenario where the set of relevant variables possibly changes over time.
 
There has also been significant effort in other fields to develop tractable
algorithms for approximating $d$ variate functions (with $d$ large) from point
queries by assuming the functions to intrinsically depend on a few variables or
parameters (cf. \cite{Devore2011,Belkin03,Coifman06,Greenshtein06} and
references within)

\subsection*{Our Contributions} Firstly, assuming the $k$-tuple
$(i_1,\dots,i_k)$ to be fixed across time but unknown to the player,
we derive an algorithm $\CABdk$ that achieves a 
regret bound 
of $O(n^{\frac{\alpha+k}{2\alpha+k}} (\log n)^{\frac{\alpha}{2\alpha+k}} C(k,d))$,
after $n$ rounds. Here $\alpha \in (0,1]$ denotes the
exponent of H\"{o}lder continuity of the reward functions. The additional factor
$C(k,d)$ which depends at most polynomially in $k$ and sub-logarithmically in
$d$ captures the uncertainty of not knowing the $k$ active coordinates. 
When $\alpha = 1$, i.e. the reward functions are Lipschitz continuous, our bound is 
nearly optimal\footnote{\label{note1} See Section \ref{sec:conclusion} for discussion on how the
$\log n$ term can be removed.} in terms of $n$ (up to the $(\log n)^{\frac{1}{2+k}}$ factor).
Note that the number of rounds $n$ after which the per-round regret $R(n)/n < c$, for any constant $0 < c < 1$, is exponential in $k$.
Hence for $k \ll d$, we do not suffer from the curse of dimensionality. The algorithm is \emph{anytime} in the sense that
$n$ is not required to be known and is a simple modification of the CAB1 algorithm \cite{Kleinberg04}.
The modification is in the manner of discretization of $[0,1]^d$ for which we
consider a probabilistic construction based on creating \textit{partitions} of
$\set{1,\dots,d}$ into $k$ disjoint subsets. The above bound holds for both the
stochastic (underlying $g$ is sampled in an i.i.d manner) and the adversarial
(underlying $g$ chosen arbitrarily at each round) models. Secondly, we extend
our results to handle the more general setting where an adversary chooses some
sequence of $k$-tuples $(\veci_t)_{t=1}^{n} = (i_{1,t},\dots,i_{k,t})_{t=1}^{n}$
before the start of plays. For this setting we derive a regret bound of
$O(n^{\frac{\alpha+k}{2\alpha+k}} (\log n)^{\frac{\alpha}{2\alpha+k}}
H[(\veci_t)_{t=1}^{n}] C(k,d))$ where $(H[\veci_t])_{t=1}^{n}$ denotes the
``hardness''\footnote{See Definition \ref{def:hardness_seq} in Section
\ref{subsec:adver_case_k1}.} of the sequence $(\veci_t)_{t=1}^{n}$. In case
$H[(\veci_t)_{t=1}^{n}] \leq S$ for some $S > 0$ known to player, this bound
improves to $O(n^{\frac{\alpha+k}{2\alpha+k}} (\log
n)^{\frac{\alpha}{2\alpha+k}} S^{\frac{\alpha}{2\alpha+k}} C(k,d))$.

\subsection*{Organization of the paper} The rest of the paper is organized as
follows. In Section \ref{sec:prob_setup} we define the problem statement
formally and outline our main results. Section \ref{sec:analysis_k1_case}
contains the analysis for the case $k=1$ for both the Stochastic and Adversarial
models. In Section \ref{sec:gen_k_analysis} we present an analysis for the
general setting where $1 \leq k \leq d$, including the construction of the
discrete strategy sets. Finally in Section \ref{sec:conclusion} we summarize our
results and provide directions for future work.

%% file: bandits_prob_setup.tex
%
\section{Problem Setup and Notation} \label{sec:prob_setup}
The compact set of strategies $S = [0,1]^d \subset \matR^d$ is available to the
player. Time steps are denoted by $\set{1,\dots,n}$. At each time step $t \in
\set{1,\dots,n}$, a reward function $r_t: S \rightarrow \matR$ is chosen by the
environment. Upon playing a strategy $\vecx_t \in [0,1]^d$, the player receives
the reward $r_t(\vecx_t)$ at time step $t$. 
\begin{assumption}
For some $k \leq d$, we assume each $r_t$ to have the \textit{low dimensional}
structure 
\begin{equation}
r_t(x_1,\dots,x_d) = g_t(x_{i_{1}},\dots,x_{i_{k}})
\end{equation}
where ($i_{1},\dots,i_{k}$) is a $k$-tuple with distinct integers $i_{j} \in
\set{1,\dots,d}$ and $g_t : [0,1]^k \rightarrow \matR$.
\end{assumption}
In other words the reward at each time $t$ depends on a fixed but unknown subset
of $k$ coordinate variables. For simplicity of notation, we denote the set of
$k$-tuples of the set $\set{1,\dots,d}$ by $\calT^d_k$ and the $\ell_2$ norm by
$\norm{\cdot}$. We assume that $k$ is known to the player, however it suffices
to know that $k$ is an upper bound for the number of active
variables\footnote{Indeed, any function that depends on $k^{\prime} \leq k$
coordinates also depends on at most $k$ coordinates.}. 
The second assumption that we make is on the smoothness property of the reward functions. 
\begin{definition} \label{def:holder_smooth_def}
A function $f: [0,1]^k \rightarrow \mathbb{R}$ is locally uniformly H\"{o}lder
continuous with constant $0 \leq L < \infty$, exponent $0 < \alpha \leq 1$, and
restriction $\delta > 0$ if we have for all $\vecu,\vecu^{\prime} \in [0,1]^k$
with $\norm{\vecu - \vecu^{\prime}} \leq \delta$ that
\begin{equation}
\abs{f(\vecu) - f(\vecu^{\prime})} \leq L \norm{\vecu - \vecu^{\prime}}^{\alpha}.
\end{equation}
We denote the class of such functions $f$ as $\calC(\alpha,L,\delta,k)$. 
\end{definition}
The function class defined in Definition \ref{def:holder_smooth_def} was also
considered in \cite{Agrawal95,Kleinberg04} and is a generalization of Lipschitz
contiuity (obtained for $\alpha=1$). We now define the two models that we
analyze in this paper. These models describe how the reward functions $g_t$ are
generated at each time step $t$.
\begin{itemize}
\item \textit{Stochastic model.} The reward functions $g_t$ are considered to be
i.i.d samples from some fixedbut unknown probability distribution over functions
$g:[0,1]^k \rightarrow \matR$. We define the expectation of the reward function
as $\bar{g}(\vecu) \ = \ \expec[g(\vecu)]$ where $\vecu \in [0,1]^k$. We require
$\bar{g}$ to belong to $\calC(\alpha,L,\delta,k)$ and note that the individual
samples $g_t$ need not necessarily be H\"{o}lder continuous. Lastly we make the
following assumption of sub-Gaussianity on the distribution from which the 
random samples $g$ are generated.
\begin{assumption} \label{assump:dist_reward_fns}
We assume that there exist constants $\zeta,s_0 > 0$ so that
\begin{equation*}
\expec[e ^{s(g(\vecu)-\bar{g}(\vecu))}] \leq e^{\frac{1}{2}\zeta^2 s^2} \quad
\forall s \in [-s_0,s_0], \vecu \in [0,1]^k.
\end{equation*}
\end{assumption}
The above assumption was considered in \cite{KleinbergPhd} for the $d=1$
situation and allows us to consider reward functions $g_t$ whose range is not
bounded. Note that the mean reward $\bar{g}$ is assumed to be H\"{o}lder
continuous and is therefore bounded as it is defined over a compact domain. The
optimal strategy $\vecx^{*}$ is then defined as any point belonging to the set
\begin{equation} \label{eq:opt_strategy_stoch}
\text{argmax}_{\vecx \in [0,1]^d} \expec[r(\vecx)] =
\text{argmax}_{\vecx \in [0,1]^d} \bar{g}(x_{i_1},\dots,x_{i_k}).
\end{equation}
\item \textit{Adversarial model.} The reward functions $g_t: [0,1]^k \rightarrow
[0,1]$ are a fixed sequence of functions in $\calC(\alpha,L,\delta,k)$ chosen arbitrarily 
by an \textit{oblivious} adversary i.e., an
adversary not adapting to the actions of the player. The optimal
strategy $\vecx^{*}$ is then defined as any point belonging to the set
\begin{equation} \label{eq:opt_strategy_adver}
\text{argmax}_{\vecx \in [0,1]^d} \sum_{t=1}^{n} r_t(\vecx) =
\text{argmax}_{\vecx \in [0,1]^d} \sum_{t=1}^{n} g_t(x_{i_{1}},\dots,x_{i_{k}}).
\end{equation}
\end{itemize}
Given the above models we measure the performance of a player over $n$ rounds in
terms of the \textit{regret} defined as
\begin{equation} \label{eq:prob_setup_regret_def}
R(n) := \sum_{t=1}^{n} \expec\left[r_t(\vecx^{*}) - r_t(\vecx_t) \right] =
\sum_{t=1}^{n} \expec\left[g_t(\vecx^{*}_{i_{1}},\dots,\vecx^{*}_{i_{k}}) -
g_t(\vecx^{(t)}_{i_{1}},\dots,\vecx^{(t)}_{i_{k}}) \right].
\end{equation}
In \eqref{eq:prob_setup_regret_def} the expectation is defined over (i) the random
choices of $g_t$ for the stochastic model and (ii) the random choice of the strategy
$\vecx_t$ by the player in the stochastic/adversarial models.

\textbf{Main results.} The main results of our work are as follows. Firstly,
assuming that the $k$-tuple $(i_1,\dots,i_k) \in \calT^d_k$ is chosen once at
the beginning of play and kept fixed thereafter, we provide in the form of
Theorem \ref{thm:main_result_upper_bound} a bound on the regret which is
$O(n^{\frac{\alpha+k}{2\alpha+k}} (\log n)^{\frac{\alpha}{2\alpha+k}} C(k,d))$
where $C(k,d) = O(\text{poly}(k) \cdot o(\log d))$. This bound holds for both the
stochastic and adversarial models and is \textit{almost optimal}\footnote{See Section \ref{sec:conclusion} 
for remarks on how the $\log n$ term can be removed from regret bounds.}. To see this,
we note that \cite{Bubeck2011ALT} showed a precise exponential lower bound of
$\Omega(n^{\frac{d+1}{d+2}})$ after $n=\Omega(2^{d})$ plays for stochastic
continuum armed bandits with $d$-variate Lipschitz continuous reward functions
defined over $[0,1]^d$. In our setting though, the reward functions depend on an
unknown subset of $k$ coordinate variables hence any algorithm after $n =
\Omega(2^k)$ plays would incur worst case regret of
$\Omega(n^{\frac{k+1}{k+2}})$ which is still mild if $k \ll d$. We see that our
upper bound matches this lower bound for the case of Lipschitz continuous reward
functions ($\alpha=1$) up to a mild factor of $(\log n)^{\frac{1}{2+k}} C(k,d)$.
We also note that the $(\log d)^{\frac{\alpha}{2\alpha+k}}$ factor in \eqref{eq:main_result_reg_up_bd}
accounts for the uncertainty in not knowing which $k$ coordinates are active
from $\set{1,\dots,d}$. 
\begin{theorem} \label{thm:main_result_upper_bound}
Given that the $k$-tuple $(i_1,\dots,i_k) \in \calT^d_k$ is kept fixed across
time but unknown to the player, the algorithm $\CABdk$ incurs a regret of
\begin{equation} \label{eq:main_result_reg_up_bd}
O\left(n^{\frac{\alpha+k}{2\alpha+k}} (\log n)^{\frac{\alpha}{2\alpha+k}}
k^{\frac{\alpha(k+6)}{2(2\alpha+k)}} (\log d)^{\frac{\alpha}{2\alpha+k}}\right)
\end{equation}
after $n$ rounds of play with high probability for both the stochastic and
adversarial models.
\end{theorem}
The above result is proven in Section \ref{subsec:gen_k_analysis_fixed} along
with a description of the $\CABdk$ algorithm which achieves this bound. The main
idea here is to discretize $[0,1]^d$ by first constructing a family of
partitions $\calA$ of $\set{1,\dots,d}$ with each partition consisting of $k$
disjoint subsets. The construction is probabilistic and the resulting $\calA$
satisfies an important property (with high probability) namely the
\textit{Partition Assumption} as described in Section \ref{sec:gen_k_analysis}.
In particular we have that $\abs{\calA}$ is $O(k e^k \log d)$ resulting in a
total of $M^k\abs{\calA}$ points for some integer $M > 0$. The discrete strategy
set is then used with a finite armed bandit algorithm such as UCB-1
\cite{Auer02} for the stochastic setting and Exp3 \cite{Auer95} for the
adversarial setting, to achieve the regret bound of Theorem
\ref{thm:main_result_upper_bound}.

Secondly we extend our results to the setting where $(i_1,\dots,i_k)$ can change
over time. Considering that an oblivious adversary chooses arbitrarily before
the start of plays a sequence of $k$ tuples $(\veci_t)_{t=1}^{n} =
(i_{1,t},\dots,i_{k,t})_{t=1}^{n}$ of \textit{hardness} (see Definition
\ref{def:hardness_seq} in Section \ref{subsec:adver_case_k1})
$H[(\veci_t)_{t=1}^{n}] \leq S$ with $S > 0$ known to the player, we show how
Algorithm $\CABdk$ can be adapted to this setting to achieve a regret bound of
$O\left(n^{\frac{\alpha+k}{2\alpha+k}} (\log n)^{\frac{\alpha}{2\alpha+k}}
S^{\frac{\alpha}{2\alpha+k}} C(k,d)\right)$. Hardness of a sequence is defined
as the number of adjacent elements with different values. In case the player has
no knowledge of $S$, the regret bound then changes to
$O(n^{\frac{\alpha+k}{2\alpha+k}} (\log n)^{\frac{\alpha}{2\alpha+k}}
H[(\veci_t)_{t=1}^{n}] C(k,d))$. Although our bound becomes trivial when
$H[(\veci_t)_{t=1}^{n}]$ is close to $n$ (as one would expect) we can still
achieve sub-linear regret when $H[(\veci_t)_{t=1}^{n}]$ is small relative to
$n$. We again consider a discretization of the space $[0,1]^d$ constructed using
the family of partitions $\calA$ mentioned earlier. The difference lies in now
using the Exp3.S algorithm \cite{Auer03} on the discrete strategy set, which in
contrast to the Exp3 algorithm is designed to control regret against arbitrary
sequences. This is described in detail in Section
\ref{subsec:gen_k_analysis_change}.

%% file: bandits_k1_analysis.tex
\section{Analysis for case $k = 1$} \label{sec:analysis_k1_case}

We first consider the relatively simple case when $k=1$ as the analysis provides intuition about the more general setting where $k \geq 1$.

\input{bandits_stoch_k1}

\input{bandits_adver_k1}

%% file: bandits_stoch_k1.tex
%
\subsection{Stochastic setting for the case $k=1$} \label{subsec:stoch_case_k1}
The functions $(r_t)_{t=1}^{n}$ are independent samples from a fixed unknown
probability distribution on functions $r:[0,1]^d \rightarrow \matR$ where
$r(\vecx) = g(x_i)$, $\vecx = (x_1,\dots,x_d)$. Here the unknown coordinate $i$
is also assumed to be drawn independently from a fixed probability distribution
on $\set{1,\dots,d}$. Thus at each $t$, the environment first randomly samples
$g_t$ from some fixed probability distribution on functions $g : [0,1]
\rightarrow \matR$ and then independently samples the active coordinate $i_t$
from some distribution on $\set{1,\dots,d}$. Note that the player has no
knowledge of $g_t$ and $i_t$ at any time $t$, only the value of $r_t$ can be
queried at some $\vecx_t \in [0,1]^d$ for all $t = 1,\dots,n$. By plugging 
$k=1$ in \eqref{eq:opt_strategy_stoch}, we have that the optimal constant strategy $\vecx^{*}$ 
is defined as any point belonging to the set:
\begin{equation}
\text{argmax}_{\vecx \in [0,1]^d} \expec[r(\vecx)] =
\text{argmax}_{\vecx \in [0,1]^d} \expec_{g,i}[g(x_i)] = \text{argmax}_{\vecx
\in [0,1]^d} \expec_i[\bar{g}(x_i)].
\end{equation}
Here $\expec_i[\cdot]$ denotes expectation with respect to the distribution over
$\set{1,\dots,d}$. We now proceed towards the discretization of $[0,1]^d$. On
account of the low dimensional structure of the reward functions $r_t$, we
consider the region 
\begin{equation*}
\calP := \set{(x,\dots,x) \in \mathbb{R}^d: x\in [0,1]}
\end{equation*}
and focus on its discretization. In particular we consider $M$ equi-spaced
points from $\calP$ (for some integer $M  > 0$ to be specified later) and then
runs a finite armed bandit algorithm on the sampled points. Given that the
player sampled $M$ points, we denote the set of sampled points by 
\begin{equation}
\calP_M := \set{\vecp_j \in [0,1]^d : \vecp_j =
\frac{j}{M}(1,\dots,1)^T}_{j=1}^{M}.
\end{equation}
Note that $r_t(\vecp_j) = g_t(j/M)$ for $\vecp_j \in \calP_M$. Over a play of
$T$ rounds, the regret $R(T)$ is then defined as
\begin{equation}
R(T) := \expec[\sum_{t=1}^{T} r_t(\vecx^{*}) - r_t(\vecx_t)] =
\expec[\sum_{t=1}^{T} g_t(x^{*}_{i_t}) - g_t(x_t)] = \expec[\sum_{t=1}^{T}
\bar{g}(x^{*}_{i_t}) - \bar{g}(x_t)]
\end{equation}
where $\vecx_t =(x_t,\dots,x_t) \in \calP_M$ and the expectation in the
rightmost term is over the random choice $i_t$ at each $t$. We now state in
Lemma \ref{lemma:stoch_k1_reg_bd} a bound on the regret $R(T)$ incurred by a
player with knowledge of the time horizon $T$ and which leverages the UCB-1
algorithm \cite{Auer02} on the finite strategy set $\calP_M$.
\begin{lemma} \label{lemma:stoch_k1_reg_bd}
Given the above setting and assuming that the UCB-1 algorithm is used along with
the strategy set $\calP_M$, we have for $M = \lceil(\frac{T}{\log
T})^{\frac{1}{2\alpha+1}}\rceil$ that $R(T) = O(T^{\frac{1+\alpha}{1+2\alpha}}
\log^{\frac{\alpha}{1+2\alpha}}(T))$.
\end{lemma}
\begin{proof}
Firstly we note that for some $\vecx^{\prime} = (x^{\prime},\dots,x^{\prime})
\in \calP_M$, $R(T)$ can be written as:
\begin{align*}
R(T) = \expec[\sum_{t=1}^{T}\bar{g}(x^{*}_{i_t}) - \bar{g}(x^{\prime})] +
\expec[\sum_{t=1}^{T}\bar{g}(x^{\prime}) - \bar{g}(x_t)] = R_1(T) + R_2(T)
\end{align*}
Now we claim that $\exists \ x^{\prime} \in \set{1/M,\dots,1}$ s.t $R_1(T) <
T/M^{\alpha}$. To see this observe that $\expec_{i_t}[\bar{g}(x^{*}_{i_t})] <
\bar{g}(x^{*}_{l_{\max}})$ where $l_{\max} = \text{argmax}_{i \in
\set{1,\dots,d}} \bar{g}(x^{*}_i).$. Since $x^{*}_{l_{max}} \in [0,1]$ hence the
claim follows by choosing $x^{\prime}$ closest to $x^{*}_{l_{max}}$. Therefore
$R_1(T)$ can be bounded as follows.
\begin{equation*}
R_1(T) < \sum_{t=1}^{T}(\bar{g}(x^{*}_{l_{max}}) - \bar{g}(x^{\prime})) <
\sum_{t=1}^{T} L \abs{x^{*}_{l_{max}} - x^{\prime}}^{\alpha} < T L M^{-\alpha}.
\end{equation*}
It remains to bound $R_2(T)$. Note here that $R_2(T)$ is bounded by the actual
regret that would be incurred on the strategy set $\set{1/M,\dots,1}$. Let
$\vecx_{*} \in \calP_M$ be optimal so that 
\begin{equation*}
\vecx_{*} = (x_{*},\dots,x_{*}) \in \text{argmax}_{\vecx \in \calP_M} \bar{r}(\vecx)
\end{equation*}
where $x_{*} = \text{argmax}_{x \in \set{1/M,\dots,1}} \bar{g}(x)$. Hence
$R_2(T) < \expec[\sum_{t=1}^{T} \bar{g}(x_{*}) - \bar{g}(x_t)]$. On account of
Assumption \ref{assump:dist_reward_fns}, it can be shown that $R_2(T) =
O(\sqrt{MT \log T})$ for UCB-1 algorithm for $M$-armed stochastic bandits (see
Theorem $3.1$, \cite{Kleinberg04}). Combining the bounds for $R_1(T)$ and
$R_2(T)$ we then have that
\begin{align}
R(T) = O\left(\sqrt{MT \log T} + \frac{T}{M^{\alpha}}\right)
\label{eq:reg_stoc_bound}
\end{align}
Finally by plugging $M = \lceil(\frac{T}{\log T})^{\frac{1}{2\alpha+1}}\rceil$
in \eqref{eq:reg_stoc_bound} we obtain $R(T) = O(T^{\frac{1+\alpha}{1+2\alpha}}
\log^{\frac{\alpha}{1+2\alpha}} T)$.
\end{proof}
Then using strategy set $\calP_M$ with $M$ defined as above, we can employ
Algorithm \ref{alg:cont_arm_band_k1} with \textbf{MAB} sub-routine being the
UCB-1 algorithm.
%
\begin{algorithm} 
\caption{Algorithm CAB1($d,1$)}\label{alg:cont_arm_band_k1}
\begin{algorithmic}
\State $T = 1$ 

\While {$T \leq n$}

\State $M = \left\lceil\left(\frac{T}{\log T}\right)^{1/(2\alpha +
1)}\right\rceil$

\State Initialize \textbf{MAB} with $\calP_M := \set{\vecp_j \in [0,1]^d:
\vecp_j = \frac{j}{M}(1,\dots,1)^T}_{j=1}^{M}$ 

\For{$t=T,\dots,\min(2T-1,n)$} 

\begin{itemize}[leftmargin=1.5cm]
\item get $\vecx_t$ from \textbf{MAB}

\item Play $\vecx_t$ and get $r_t(\vecx_t)$

\item Feed $r_t(\vecx_t)$ back to \textbf{MAB}
\end{itemize}

\EndFor 

\State $T = 2T$

\EndWhile

\end{algorithmic}
\end{algorithm}
The above algorithm basically employs the discrete set $\calP_M$ with Algorithm
CAB1, proposed in \cite{Kleinberg04}. The main idea is to split the unknown time
horizon into intervals of varying size with each interval being twice as large
as the preceding one (doubling trick). Since the regret over an interval of $T$
time steps (with $T$ now known to the player) is
$O(T^{\frac{1+\alpha}{1+2\alpha}} \log^{\frac{\alpha}{1+2\alpha}} T)$, we have
that the overall regret over $n$ plays is $O(n^{\frac{1+\alpha}{1+2\alpha}}
\log^{\frac{\alpha}{1+2\alpha}} n)$. Note that the regret incurred is
independent of the dimension $d$ and is also optimal up to a $\log$-factor since
a $\Omega(n^{\frac{1+\alpha}{1+2\alpha}})$ bound is known \cite{Kleinberg04} for
the $d=1$ scenario.

%% file: bandits_adver_k1.tex
%
\subsection{Adversarial setting for the case $k=1$} \label{subsec:adver_case_k1}
In this setting, the reward functions $(r_t)_{t=1}^{n}$ are chosen beforehand by
an \textit{oblivious} adversary i.e., the adversary does not select $r_t$ based
on the past sequence of action-response pairs:
$(\vecx_1,r_1(\vecx_1),\vecx_2,r_2(\vecx_2),\dots,\vecx_{t-1},r_{t-1}(\vecx_{t-1
}))$. The reward function $r_t : [0,1]^d \rightarrow [0,1]$ has the low
dimensional structure $r_t(\vecx) = g_t(x_{i})$, where $i \in \set{1,\dots,d}$
is chosen once at the beginning of plays and kept fixed thereafter. Hence at
each $t$ the adversary chooses $g_t \in \calC(\alpha,\delta,L,1)$ arbitrarily
but oblivious to the plays of the algorithm. By plugging $k=1$ in
\eqref{eq:opt_strategy_adver} we have that the optimal constant strategy $\vecx^{*}$ 
is defined as any point belonging to the set:
\begin{equation*}
\text{argmax}_{\vecx \in [0,1]^d} \sum_{t=1}^{n} r_t(\vecx) =
\text{argmax}_{\vecx \in [0,1]^d} \sum_{t=1}^{n} g_t(x_{i}).
\end{equation*}
The regret incurred by the algorithm over a play of $T$ rounds is defined as
\begin{equation*}
R(T) := \expec[\sum_{t=1}^{T} r_t(\vecx^{*}) - r_t(\vecx_t)] =
\expec[\sum_{t=1}^{T} g_t(x^{*}_{i}) - g_t(x_{i})].
\end{equation*}
Note that the expectation here is over the random choice of the algorithm at
each $t$, the adversary is assumed to be deterministic without loss of
generality. To start off, we first consider the scenario where the adversary
chooses some active coordinate at the beginning of play and fixes it across
time. The bound on regret $R(T)$ is shown for a player who employs the Exp3
algorithm \cite{Auer95} on the discrete set $\calP_M$ (defined in Section
\ref{subsec:stoch_case_k1}) for an appropriate value of sample size $M$.
\begin{lemma} \label{lemma:adver_k1}
Consider the scenario where the active coordinate is fixed across time, i.e.
$i_1 = i_2 = \dots = i$ and unknown to the player. Then by employing the Exp3
algorithm along with the strategy set $\calP_M$ for $M = \lceil(\frac{T}{\log
T})^{\frac{1}{2\alpha+1}}\rceil$ we have that $R(T) =
O(T^{\frac{1+\alpha}{1+2\alpha}} \log^{\frac{\alpha}{1+2\alpha}}(T))$.
\end{lemma}
\begin{proof}
For some $\vecx^{\prime} \in \calP_M$ we can re-write $R(T)$ as $R(T) = R_1(T) +
R_2(T)$ where
\begin{align*}
R_1(T) = \expec[\sum_{t=1}^{T} r_t(\vecx^{*}) - r_t(\vecx^{\prime})] =
\expec[\sum_{t=1}^{T} g_t(x_i^{*}) - g_t(x^{\prime})], \\
R_2(T) = \expec[\sum_{t=1}^{T} r_t(\vecx^{\prime}) - r_t(\vecx_t)] =
\expec[\sum_{t=1}^{T} g_t(x^{\prime}) - g_t(x_t)].
\end{align*} 
Now as $x_i^{*} \in [0,1]$, $\exists x^{\prime} \in \set{1/M,\dots,1}$ closest
to $x_i^{*}$ so that $\abs{x_i^{*}-x^{\prime}} < 1/M$. Choosing this
$\vecx^{\prime}$ we obtain $R_1(T) < \sum_{t=1}^{T} \abs{x_i^{*} -
x^{\prime}}^{\alpha} < T M^{-\alpha}$. Furthermore we can bound $R_2(T)$ as
follows.
\begin{equation*}
R_2(T) = \expec[\sum_{t=1}^{T} g_t(x^{\prime}) - g_t(x_t)] <
\expec[\sum_{t=1}^{T} g_t(x_{*}) - g_t(x_t)]
\end{equation*}
where $x_{*} := \text{argmax}_{x \in \set{1/M,\dots,1}} \sum_{t=1}^{T} g_t(x)$.
In other words $x_{*}$ is ``optimal'' out of $\set{1/M,\dots,1}$. Hence by using
the Exp3 algorithm on $\calP_M$ we get: $R_2(T) = O(\sqrt{TM \log M})$. Lastly
for $M = \lceil(\frac{T}{\log T})^{\frac{1}{2\alpha+1}}\rceil$ we have that
$R(T) = O(T^{\frac{1+\alpha}{1+2\alpha}}\log^{\frac{\alpha}{1+2\alpha}}T)$.
\end{proof}
Using strategy set $\calP_M$ with $M$ defined as above, we can employ Algorithm
\ref{alg:cont_arm_band_k1} with \textbf{MAB} sub-routine being the Exp3
algorithm. As the regret in the inner loop is $O(T^{\frac{1+\alpha}{1+2\alpha}}
\log^{\frac{\alpha}{1+2\alpha}} T)$, we have as a consequence of the doubling
trick that the overall regret over $n$ plays is
$O(n^{\frac{1+\alpha}{1+2\alpha}} \log^{\frac{\alpha}{1+2\alpha}} n)$. We also
see that the bound is independent of the dimension $d$.
%
\begin{remark}
Another way to look at the above setting, where the active coordinate is
constant over time is the following. Let $x_{max} \in [0,1]$ be such that
$x_{max} \in \text{argmax}_{x \in [0,1]} \sum_{t=1}^{T} g_t(x)$. Then provided
that $i_1 = \dots = i_t = i$, we have that $\vecx^{*} = x_{max}(1,\dots,1)^T$
belongs to the set of optimal solutions, i.e. points in $[0,1]^d$ which maximize
$\sum_{t}g_t(x_i)$. Therefore we can play strategies from the set $\calP_M$ and
employ Algorithm \ref{alg:cont_arm_band_k1} to achieve regret independent of
$d$.

We also briefly remark on the scenario where the reward function $g_t$ remains
unchanged across time, i.e. $g_1 = \dots = g_t = g$ while the active coordinate
$i_t$ changes arbitrarily. For any $x_{max} \in \text{argmax}_{x \in [0,1]} g(x)$, 
clearly $\vecx^{*} = x_{max}(1,\dots,1)^T$ maximizes $\sum_{t}g(x_{i_t})$. 
Hence in this case too, we can play strategies from the set $\calP_M$ and 
employ Algorithm \ref{alg:cont_arm_band_k1} to achieve regret independent of $d$.
\end{remark}
\subsection*{Lower bound on regret for arbitrary change of active coordinate.}
In Lemma \ref{lemma:adver_k1} we had considered the setting where the active
coordinate remains fixed over time. We now show that in case the active
coordinate is allowed to change arbitrarily, then the worst-case regret incurred
by any algorithm playing strategies only from the region $\calP :=
\set{(x,\dots,x) \in \mathbb{R}^d: x\in [0,1]}$ over $T$ rounds will be
$\Omega(T)$. To this end we construct the adversary $\adver$ as follows. Before
the start of play, $\adver$ chooses uniformly at random a 2-partition
$(A_1,A_2)$ of $\set{1,\dots,d}$ which is then kept fixed across time. Note that
there are in total $2^{d}-2$ \textit{ordered} 2-partitions of $\set{1,\dots,d}$
where $A_1,A_2 \neq \phi$ so that $A_1$ denotes the first subset and $A_2$ the
second subset \footnote{For example, for $d=3$ we have
$\set{(\set{1,2},\set{3}),(\set{3},\set{1,2}),(\set{1,3},\set{2}),(\set{2},\set{
1,3}),(\set{2,3},\set{1}),(\set{1},\set{2,3})}$ as the set of possible
2-partitions with ordering and with $A_1,A_2 \neq \phi$.}. At each $t$, $\adver$
then randomly selects w.p $1/2$ either $h_1$ or $h_2$ where

\begin{equation*}
h_1(x) = \left\{
\begin{array}{rl}
u(x) \quad ; & x \in [0,1/2] \\
0 \quad ; & x \in [1/2,1].
\end{array} \right .
\text{and} \quad h_2(x) = \left\{
\begin{array}{rl}
0 \quad ; & x \in [0,1/2] \\
v(x) \quad ; & x \in [1/2,1].
\end{array} \right .
\end{equation*}

and sets it as $g_t$. Here $u,v$ are considered to be $C^{\infty}$ functions
taking values in $[0,1]$ with supports $[0,1/2]$ and $[1/2,1]$ respectively.
Constraining $u(0) = u(1/2) = v(1/2) = v(1) = 0$, we assume $u$ and $v$ to be
uniquely maximized at $a \in (0,1/2)$ and $b \in (1/2,1)$ repsectively, so that
$u(a) = v(b) = 1$. Finally, if $h_1$ was selected then \adver \ samples the
active coordinate $i_t \in_{u.a.r} A_1$ else it samples $i_t \in_{u.a.r} A_2$.
We now prove the lower bound on the regret of all algorithms constrained to play
strategies from the region $\set{(x,\dots,x) \in \mathbb{R}^d: x\in [0,1]}$
against $\adver$, formally in Lemma \ref{lemma:adver_k1_line_lower_bd}.

\begin{lemma} \label{lemma:adver_k1_line_lower_bd}
Consider the adversarial setting with $k=1$ where,
\begin{enumerate}
\item the adversary is allowed to choose both $g_t \in
\mathcal{C}(\alpha,L,\delta,1)$ for some positive constants $\alpha,L,\delta$
and $i_t \in \set{1,\dots,d}$ arbitrarily at each $t$, and
\item the player is only allowed to play strategies from $\calP :=
\set{(x,\dots,x) \in \mathbb{R}^d: x\in [0,1]}$
\end{enumerate}
We then have that the worst-case regret $R(T)$ of any algorithm over $T$ rounds
satisfies $R(T) \geq T/2$.
\end{lemma}
\begin{proof}
We consider the adversary $\adver$ whose construction was described previously.
Any randomized playing strategy is equivalent to an a-priori random choice from
the set of all deterministic strategies. Since \adver \ is oblivious to the
actions of the player, it suffices to consider an upper bound on expected gain
holding true for all deterministic strategies. Furthermore we consider all
expectations to be conditioned on the event that the 2 partition $(A_1,A_2)$ was
chosen at the beginning by $\adver$.

Consider any deterministic algorithm $\algo$. Here $\algo :
((\vecx_1,r_1(\vecx_1)),(\vecx_2,r_2(\vecx_2)), \dots
,(\vecx_{t-1},r_{t-1}(\vecx_{t-1}))) \rightarrow \vecx_t$ is a deterministic
mapping that maps the past $(t-1)$ plays/responses to a fixed strategy $\vecx_t$
at time $t$. Therefore we get

\begin{equation*}
\expec[g_t(x_{i_t})] = \frac{1}{2}\left[\sum_{i_t \in A_1} \frac{1}{\abs{A_1}}
h_1(x_{i_t})\right] + \frac{1}{2}\left[\sum_{i_t \in A_2} \frac{1}{\abs{A_2}}
h_2(x_{i_t})\right]
\end{equation*}

Denoting $a = \text{argmax}_{x \in [0,1]} h_1(x)$ and $b = \text{argmax}_{x \in
[0,1]} h_2(x)$ we clearly have that $\vecx^{*} \in [0,1]^d$ where

\begin{equation*}
(\vecx^{*})_i = \left\{
\begin{array}{rl}
a \quad ; & i \in  A_1 \\
b \quad ; & i \in A_2.
\end{array} \right .
\end{equation*}

uniquely maximizes $\expec[g_t(x_{i_t})]$ for each $t = 1,\dots,T$. In
particular,

\begin{equation} \label{eq:opt_str_adv}
\expec[g_t(x^{*}_{i_t})] = \frac{1}{2}\left[\sum_{i \in A_1} \frac{1}{\abs{A_1}}
h_1(a)\right] + \frac{1}{2}\left[\sum_{i \in A_2} \frac{1}{\abs{A_2}}
h_2(b)\right] = \frac{1}{2}[h_1(a) + h_2(b)] = 1
\end{equation} 

Using \eqref{eq:opt_str_adv} and denoting $\vecx_t = (x_t,\dots,x_t) \in \calP$
as the strategy played at time $t$, we have 

\begin{equation*}
R((A_1,A_2),T) = \expec[\sum_t r_t(\vecx^{*}) - r_t(\vecx_t)] = \expec[\sum_t (1
- g_t(x_t))]
\end{equation*}

where $R((A_1,A_2),T)$ denotes the regret incurred by $\algo$ when $(A_1,A_2)$
was selected as the 2 partition at the beginning by $\adver$. Finally, since

\begin{equation*}
\expec[g_t(x_t)] = \left\{
\begin{array}{rl}
\frac{1}{2} h_1(x_t) \quad ; & x_t \in [0,1/2] \\
\frac{1}{2} h_2(x_t) \quad ; & x_t \in [1/2,1] 
\end{array} \right .
\end{equation*}

we have that $\expec[g_t(x_t)] \leq 1/2$ for any deterministic algorithm at time
$t$. Hence $R((A_1,A_2),T) \geq T/2$. 
\end{proof}
\subsection*{Upper bound on regret for restricted changes of active coordinate.}
We now consider the adversarial setting where the active coordinate are allowed
to change across time but not in an arbitrary manner as before. To start off we
have at each time $1 \leq t \leq n$ that each reward function is of the form
$r_t(\vecx) = g_t(x_{i_t})$ where $i_t \in \set{1,\dots,d}$ denotes the active
coordinate. We consider that both $(g_t)_{t=1}^{n}$ and $(i_t)_{t=1}^{n}$ are
chosen before the start of play by an adversary, oblivious to the actions of the
player. However we assume that the sequence of active coordinates
$(i_t)_{t=1}^{n}$ is not ``hard'' implying that it contains a few number of
consecutive pairs (relative to the number of rounds $n$) with different values.
We now formally present the definition of hardness of a sequence.

\begin{definition} \label{def:hardness_seq}
For any set $\calB$ we define the hardness of the sequence $(b_1,b_2,\dots,b_n)$
by:
\begin{equation}
H[b_1,\dots,b_n] := 1 + \abs{\set{1 \leq l < n : b_l \neq b_{l+1}}}
\end{equation}
where each $b_l \in \calB$.
\end{definition} 

The above definition is borrowed from Section 8 in \cite{Auer03} where the
authors considered the non-stochastic multi-armed bandit problem, and employed
the definition to characterize the hardness of a sequence of actions. For our
purposes, we assume that the adversary chooses a sequence of active coordinates
$(i_1,\dots,i_n)$ such that $H[i_1,\dots,i_n] \leq S$ where $S$ is known to the
player. We now proceed to show how a slight modification of Algorithm
\ref{alg:cont_arm_band_k1} can be used to achieve a bound on regret, by playing
strategies along the line $\calP$ defined earlier. The main idea here is to
observe that for any sequence $(i_1,\dots,i_n)$ of hardness at most $S$, we will
have at most $S+1$ ``constant'' sub-sequences, i.e. sequences with consecutive
coordinates being equal. Consider the optimal strategy $\vecx^{*}$ where
\begin{equation*}
\vecx^{*} \in \text{argmax}_{\vecx \in [0,1]^d} \sum_{t=1}^{n} g_t(x_{i_t}).
\end{equation*}
We then have that the sequence $(x^{*}_{i_1},\dots,x^{*}_{i_n})$ has hardness at
most $S$. Alternatively, this means that the sequence
$(x^{*}_{i_t},\dots,x^{*}_{i_t})_{t=1}^{n}$ which represents the sequence of
optimal points on the line $\calP$, is at most $S$-hard. Therefore if we
restrict ourselves to playing strategies from $\calP$, the problem reduces to a
one dimensional continuum armed bandit problem where the the players actions
over time, are required to be close to the ``optimal sequence'' along $\calP$,
i.e. $(x^{*}_{i_t},\dots,x^{*}_{i_t})_{t=1}^{n}$. In fact by playing strategies
from the discrete set of equi-spaced points $\calP_M$ as earlier (for some
integer $M > 0$), we obtain a multi-armed adversarial bandit problem where the
players regret is measured against a $S$ hard sequence of actions. For such a
bandit problem, the algorithm Exp3.S was proposed in \cite{Auer03}, which can
now be employed in place of the \textbf{MAB} routine in Algorithm
\ref{alg:cont_arm_band_k1}. With this in mind, we present in the following lemma
an upper bound on the regret $R(T)$ incurred by a player playing strategies from
the set $\calP_M$ over $T$ rounds.
\begin{lemma} \label{lemma:adver_k1_changing}
Consider the scenario where the sequence of active coordinates $(i_1,\dots,i_n)$
is at most $S$ hard. Then by employing the Exp3.S algorithm along with the
strategy set $\calP_M$ for $M = \lceil(\frac{T}{S\log
T})^{\frac{1}{2\alpha+1}}\rceil$ we have that $R(T) =
O(T^{\frac{1+\alpha}{1+2\alpha}} (\log T)^{\frac{\alpha}{1+2\alpha}}
S^{\frac{\alpha}{2\alpha+1}})$.
\end{lemma}
\begin{proof}
For some $\vecx^{\prime}_t = (x^{\prime}_t,\dots,x^{\prime}_t) \in \calP_M$ we
can re-write $R(T)$ as $R(T) = R_1(T) + R_2(T)$ where

\begin{align*}
R_1(T) = \expec[\sum_{t=1}^{T} r_t(\vecx^{*}) - r_t(\vecx_t^{\prime})] =
\expec[\sum_{t=1}^{T} g_t(x_{i_t}^{*}) - g_t(x^{\prime}_t)], \\
R_2(T) = \expec[\sum_{t=1}^{T} r_t(\vecx^{\prime}_t) - r_t(\vecx_t)] =
\expec[\sum_{t=1}^{T} g_t(x^{\prime}_t) - g_t(x_t)].
\end{align*} 

Here $\vecx_t = (x_t,\dots,x_t) \in \calP_M$ denotes the strategy played at time
$t$. Now for each $x_{i_t}^{*}$, $\exists x^{\prime}_t \in \set{1/M,\dots,1}$
closest to $x_{i_t}^{*}$ so that $\abs{x_{i_t}^{*}-x^{\prime}_t} < 1/M$.
Choosing this $\vecx^{\prime}_t = (x^{\prime}_t,\dots,x^{\prime}_t) \in \calP_M$
we obtain $R_1(T) < \sum_{t=1}^{T} L \abs{x_{i,t}^{*} - x^{\prime}_t}^{\alpha} <
T L M^{-\alpha}$. It remains to bound $R_2(T)$. To this end note that the
sequence $(\vecx^{\prime}_t)_{t=1}^{n}$ is also at most $S$ hard, hence the
problem has reduced to a $\abs{\calP_M}$ armed adversarial bandit problem with a
$S$ hard optimal sequence of plays against which the regret of the player is to
be bounded. This is accomplished by using the Exp3.S algorithm of \cite{Auer03}.
In particular from Corollary 8.3,\cite{Auer03} we have that $R_2(T) = O(\sqrt{S
M T \log(MT)})$. This gives us the following expression for $R(T)$:
\begin{equation*}
R(T) = O(T M^{-\alpha} + \sqrt{S M T \log(MT)}).
\end{equation*}
Finally upon plugging in the choice $M =  \lceil(\frac{T}{S\log
T})^{\frac{1}{2\alpha+1}}\rceil$ in $R(T)$ we obtain the stated bound.
\end{proof}
Since Algorithm \ref{alg:cont_arm_band_k1} divides the unknown time horizon $n$
into intervals of size $T,2T,4T,\dots$ it follows that the overall regret $R(n)$
after $n$ rounds of play is $R(n) = O(n^{\frac{1+\alpha}{1+2\alpha}} (\log
n)^{\frac{\alpha}{1+2\alpha}} S^{\frac{\alpha}{2\alpha+1}})$. 
\begin{remark}
In case the player does not know $S$, then a regret of 
\begin{equation*}
R(n) = O(n^{\frac{1+\alpha}{1+2\alpha}} (\log n)^{\frac{\alpha}{1+2\alpha}}
H[(i_t)_{t=1}^{n}])
\end{equation*}
would be incurred by Algorithm \ref{alg:cont_arm_band_k1} with the \textbf{MAB}
routine being the Exp3.S algorithm and for the choice
\begin{equation*}
M = \left\lceil\left(\frac{T}{\log T}\right)^{\frac{1}{2\alpha+1}}\right\rceil
\end{equation*}
This can be verified easily along the lines of the proof of Lemma
\ref{lemma:adver_k1_changing} by noting that on account of Corollary 8.2 of
\cite{Auer03}, we have $R_2(T) = O(H[(i_t)_{t=1}^{n}] \sqrt{M T \log(MT)})$.
\end{remark}

%% file: bandits_gen_k_analysis.tex
%
\section{Analysis for general case, with $k$ active coordinates}
\label{sec:gen_k_analysis}
We now consider the general situation where the reward function $r_t$ consists
of $k$ active coordinates with $1 \leq k \leq d$. Note that for our analysis, we
assume that $k$ is known to the player. However knowing a bound for $k$ also
suffices. The discretized strategy set that we employ for our purposes here has
a more involved structure since $k$ can now be greater than one.  
\subsection*{Partition Assumption.} The very core of our analysis involves the usage of a specific family of partitions $\calA$
of $\set{1,\dots,d}$ where each $\matA \in \calA$ consists of $k$ disjoint subsets $(A_1,\dots,A_k)$.
In particular we require $\calA$ to satisfy an important property namely the partition assumption below.
\begin{definition}
A family of partitions $\calA$ of $\set{1,\dots,d}$ into $k$ disjoint subsets is said to satisfy the 
partition assumption if for any $k$ distict integers  $i_1,i_2,\dots,i_k \in
\set{1,\dots,d}$ there exists a partition $\matA = (A_1,\dots,A_k)$ in $\calA$ such that each set in
$\matA$ contains exactly one of $i_1,i_2,\dots,i_k$.
\end{definition}
The above definition is known as \textit{perfect hashing} in theoretical computer science and is widely used such
as in finding juntas \cite{Devore2011,Mossel03}. Two natural questions that now arise are:
\begin{itemize}
\item How large should $\calA$ be to guarantee the Partition Assumption?
\item How can one efficiently construct such a family of partitions $\calA$?
\end{itemize}
\subsection*{Constructing family of partitions $\calA$}
There is a well known probabilistic construction for the family of partitions $\calA$
satisfying the partition assumption \cite{Devore2011}. In \cite{Devore2011} the
authors made use of such a family of partitions for the problem of \textit{approximating} Lipschitz 
functions intrinsically depending on a subset of the coordinate variables. We describe the probabilistic 
construction outlined in Section $5$ of \cite{Devore2011}, below.

We proceed by drawing balls labeled $1,\dots,k$ 
uniformly at random, with replacement $d$ times. If a ball with label $j$ is drawn at 
the $r^{th}$ time, then we store the integer $r$ in $A_j$. Hence after $d$ rounds 
we would have put each integer from $\set{1,\dots,d}$ into one of the sets $A_1,\dots,A_k$.
Note that some sets might be empty. On performing $m$ independent trials of this experiment, we
obtain a family $\calA$ of $m$ partitions generated independently at random. It
remains to be shown how big $m$ should be such that $\calA$ satisfies the
partition assumption.

To see this, consider a fixed $k$ tuple $S \in {[d]
\choose k}$. The probability that a random partition $\matA$ separates the
elements of $S$ into distinct sets $A_1,\dots,A_k$ is $p =
\frac{k!}{k^k}$. Hence the probability that none of the $m$ partitions separate
$S$ is $(1-p)^m$.  Since we have $d \choose k$ tuples we thus have from the
union bound that each $k$ tuple $S \in {[d] \choose k}$ is separated by a
partition with probability at least $1 - {d \choose k} (1-p)^m$. Therefore to
have a non zero probability of success, $m$ should be large enough to ensure
that ${d \choose k} (1-p)^m < 1$. This is derived by making use of the following
series of inequalities:
\begin{align*}
{d \choose k} \left(1-\frac{k!}{k^k}\right)^m \leq {d \choose k}
\left(1-\frac{\sqrt{2\pi k}}{e^k}\right)^m \leq d^{k}(1 - e^{-k})^m \leq
d^{k}e^{-me^{-k}}
\end{align*}
where we use Stirling's approximation in the first inequality and used the fact
$(1 - x^{-1})^x \leq e^{-1}$ when $x > 1$ in the third inequality. It follows
that if $m \geq 2k e^{k} \log d$ we then have that $\calA$ meets the partition
assumption with probability at least $1 - d^{-k}$. 

Note that the size of the family obtained is \textit{nearly optimal} - it is known 
that the size of any such family is $\Omega(e^k \log d / \sqrt{k})$ \cite{Fredman84,Korner88,Nilli94}.
On the other hand there also exist efficient (i.e. take time $poly(d,k)$) \textit{deterministic} constructions 
for such families of partitions with the size of the family being $O(k^{O(\log k)} e^k \log d)$ \cite{Naor95}. 
For the convenience of the reader we outline the details of the scheme from \cite{Naor95} in Appendix \ref{sec:det_constr_hash_functions}.
We consider for our purposes in this paper the probabilistic construction of the family $\calA$ due to the smaller
size of the resulting family.    
\subsection*{Constructing strategy set $\calP_M$ using $\calA$.} Say we are
given a family of partitions $\calA$ satisfying the partition assumption. Then
using $\calA$ we construct the discrete set of strategies $\calP_M \in [0,1]^d$
for some fixed integer $M > 0$ as follows.
\begin{equation} \label{eq:strat_set_gen_k}
\calP_M := \set{\frac{1}{M} \sum_{j=1}^{k} \alpha_j \chi_{\matA_j}; \alpha_j \in
\set{1,\dots,M}, (\matA_1,\dots,\matA_k) \in \calA} \subset [0,1]^d
\end{equation}
The above set of points was also employed in \cite{Devore2011} for the function approximation problem.
Note that a strategy $P = \frac{1}{M} \sum_{j=1}^{k} \alpha_j \chi_{\matA_j}$
has coordinate value $\frac{1}{M}\alpha_j$ at each of the coordinate indices in
$A_j$. Therefore we see that for each partition $\matA \in \calA$ we have $M^k$
strategies implying a total of $M^k \abs{\calA}$ strategies in $\calP_M$. 
\subsection*{Projection property.} An important property of the strategy set
$\calP_M$ is the following. Given any $k$ tuple of distinct indices
$(i_1,\dots,i_k)$ with $i_j \in \set{1,\dots,d}$ and any integers $1 \leq
n_1, \dots, n_k \leq M$, there is a strategy $\vecx \in \calP_M$ such
that
\begin{equation*}
 (x_{i_1},\dots,x_{i_k}) = \left(\frac{n_1}{M},\dots,\frac{n_k}{M}\right).
\end{equation*}
 To see this, one can simply take a partition $\matA = (A_1,\dots,A_k)$ from
$\calA$ such that each $i_{j}$ is in a different set $A_j$ for $j=1,\dots,k$.
Then setting appropriate $\alpha_j = n_{j}$ when $i_{j} \in A_j$ we get that
coordinate $i_{j}$ of $\vecx$ has the value $n_{j}/M$.
%
\subsection{Analysis when $k$ active coordinates are fixed across time}
\label{subsec:gen_k_analysis_fixed}
We now describe our Algorithm $\CABdk$ and provide bounds on its regret when
$(i_1,\dots,i_k)$ is fixed across time. Note that the outer loop is a standard
doubling trick which is used as the player has no knowledge of the time horizon
$n$. Observe that before the start of the inner loop of duration $T$, the player
constructs the finite strategy set $\calP_M$, where $M$ increases progressively
with $T$. Within the inner loop, the problem reduces to a finite armed bandit
problem. The \textbf{MAB} routine can be any standard multi-armed bandit
algorithm such as UCB-1 (stochastic model) or Exp3 (adversarial model). The main
idea is that for increasing values of $M$, we would have for any $\vecx^{*}$ and
any $(i_1,\dots,i_k)$ the existence of an arbitrarily close point to
$(x^{*}_{i_1},\dots,x^{*}_{i_k})$ in $\calP_M$. This follows from the projection
property of $\calP_M$. Coupled with the H\"{o}lder continuity of the reward
functions this then ensures that the \textbf{MAB} routine progressively plays
strategies closer and closer to $\vecx^{*}$ leading to a bound on regret. The
algorithm is motivated by the CAB1 algorithm \cite{Kleinberg04},
however unlike the equi-spaced sampling done in CAB1 we consider a
probabilistic construction of the discrete set of sampling points based on
partitions of $\set{1,\dots,d}$.
%
\begin{algorithm} 
\caption{Algorithm CAB($d,k$)} \label{alg:cont_arm_band_gen_k}
\begin{algorithmic}
\State $T = 1$

\State Construct family of partitions $\calA$ 

\While{$T \leq n$} 

\State $M = \left\lceil\left(k^{\frac{\alpha-3}{2}} e^{-\frac{k}{2}} (\log
d)^{-\frac{1}{2}} \sqrt{\frac{T}{\log T}}\right)^{2/(2\alpha + k)}\right\rceil$ 

\begin{itemize}
\item Create $\calP_M$ using $\calA$

\item Initialize \textbf{MAB} with $\calP_M$
\end{itemize}

\For{$t=T,\dots,\min(2T-1,n)$} 

\begin{itemize}[leftmargin=1.5cm]
\item get $\vecx_t$ from \textbf{MAB}

\item Play $\vecx_t$ and get $r_t(\vecx_t)$

\item Feed $r_t(\vecx_t)$ back to \textbf{MAB}
\end{itemize}
 
\EndFor

\State $T = 2T$

\EndWhile

\end{algorithmic}

\end{algorithm}

\noindent We now present in the following lemma the regret bound incurred within
an inner loop of duration $T$. 
\begin{lemma} \label{lemma:stoch_adv_k_fix_up_bd}
Given that $(i_1,\dots,i_k)$ is fixed across time then if the strategy set
$\calP_M$ is used with:
\begin{enumerate}
\item the UCB-1 algorithm for the stochastic setting or,
\item the Exp3 algorithm for the adversarial setting, 
\end{enumerate}
we have for the choice $M = \left\lceil\left(k^{\frac{\alpha-3}{2}}
e^{-\frac{k}{2}} (\log d)^{-\frac{1}{2}} \sqrt{\frac{T}{\log
T}}\right)^{\frac{2}{2\alpha+k}}\right\rceil$ that the regret incurred by the
player after $T$ rounds is given by
\begin{equation*}
R(T) = O\left(T^{\frac{\alpha+k}{2\alpha+k}} (\log T)^{\frac{\alpha}{2\alpha+k}}
k^{\frac{\alpha(k+6)}{2(2\alpha+k)}} (\log d)^{\frac{\alpha}{2\alpha+k}} \right).
\end{equation*}
\end{lemma}
\begin{proof}
For some $\vecx^{\prime} \in \calP_M$ we can split $R(T)$ into $R_1(T) + R_2(T)$
where:
\begin{align}
R_1(T) = \sum_{t=1}^{T} \expec[g_t(x^{*}_{i_{1}},\dots,x^{*}_{i_{k}}) -
g_t(x^{\prime}_{i_{1}},\dots,x^{\prime}_{i_{k}})], \label{eq:r1_regret_fix_tup}
\\
R_2(T) = \sum_{t=1}^{T} \expec[g_t(x^{\prime}_{i_{1}},\dots,x^{\prime}_{i_{k}})
- g_t(x^{(t)}_{i_{1}},\dots,x^{(t)}_{i_{k}})]. \label{eq:r2_regret_fix_tup}
\end{align}
For the $k$ tuple $(i_{1},\dots,i_{k}) \in \calT^d_k$, there exists
$\vecx^{\prime} \in \calP_M$ with $x^{\prime}_{i_1} = \frac{\alpha_1}{M},\dots,
x^{\prime}_{i_k} = \frac{\alpha_k}{M}$ where $\alpha_1,\dots,\alpha_k$ are such
that $\abs{\alpha_j/M - x^{*}_{i_{j}}} < (1/M)$. This follows from the
projection property of $\calA$. On account of the H\"{o}lder continuity of
reward functions we then have that
\begin{align*}
\expec[g_t(x^{*}_{i_{1}},\dots,x^{*}_{i_{k}}) -
g_t(x^{\prime}_{i_{1}},\dots,x^{\prime}_{i_{k}})] < L
\left(\left(\frac{1}{M}\right)^2 k \right)^{\alpha/2}.
\end{align*}
In other words, $R_1(T) = O(T k^{\alpha/2} M^{-\alpha})$. In order to bound
$R_2(T)$, we note that the problem has reduced to a $\abs{\calP_M}$-armed bandit
problem. Specifically we note from \eqref{eq:r2_regret_fix_tup} that we are
comparing against a sub-optimal strategy $\vecx^{\prime}$ instead of the optimal
one in $\calP_M$. Hence $R_2(T)$ can be bounded by using existing bounds for
finite-armed bandit problems.  Now for the stochastic setting we can employ the
UCB-1 algorithm \cite{Auer02} and play at each $t$ a strategy $\vecx_t \in
\calP_M$. In particular, on account of Assumption \ref{assump:dist_reward_fns},
it can be shown that $R_2(T) = O(\sqrt{\abs{\calP_M} T \log T})$ (Theorem $3.1$,
\cite{Kleinberg04}). For the adversarial setting we can employ the Exp3
algorithm \cite{Auer95} so that $R_2(T) = O(\sqrt{\abs{\calP_M} T \log
\abs{\calP_M}})$. Combining the bounds for $R_1(T)$ and $R_2(T)$ and recalling
that $\abs{\calP_M} = O(M^k k e^k \log d)$ we obtain:
\begin{align} 
R(T) &=  O(T M^{-\alpha} k^{\alpha/2} + \sqrt{M^k k e^k \log d \ T \log T})
\quad \text{(stochastic)}, \label{eq:gen_stoch_regret_bound_fix} \\
\text{and} \ R(T) &= O(T M^{-\alpha} k^{\alpha/2} + \sqrt{M^k k e^k \log d \ T
\log(M^k k e^k \log d)}). \quad \text{(adversarial)}
\label{eq:gen_adver_regret_bound_fix}
\end{align}
Plugging $M = \left\lceil\left(k^{\frac{\alpha-3}{2}} e^{-\frac{k}{2}} (\log
d)^{-\frac{1}{2}} \sqrt{\frac{T}{\log
T}}\right)^{\frac{2}{2\alpha+k}}\right\rceil$ in
\eqref{eq:gen_stoch_regret_bound_fix} and \eqref{eq:gen_adver_regret_bound_fix}
we obtain the stated bound on $R(T)$ for the respective models. 
\end{proof}
Lastly equipped with the above bound we have that the regret incurred by
Algorithm 1 over $n$ plays is given by
\begin{equation*}
\sum_{i=0,T=2^i}^{i = \log n} R(T) = O\left(n^{\frac{\alpha+k}{2\alpha+k}} (\log
n)^{\frac{\alpha}{2\alpha+k}} k^{\frac{\alpha(k+6)}{2(2\alpha+k)}}
(\log d)^{\frac{\alpha}{2\alpha+k}}\right).
\end{equation*}
\subsection{Analysis when $k$ active coordinates change across time.}
\label{subsec:gen_k_analysis_change}
We now  consider a more general \emph{adversarial} setting where the the active
$k$ tuple is allowed to change over time. Formally this means that the reward
functions $(r_t)_{t=1}^{n}$ now have the form $r_t(x_1,\dots,x_d) =
g_t(x_{i_{1,t}},\dots,x_{i_{k,t}})$ where $(i_{1,t},\dots,i_{k,t})_{t=1}^{n}$
denotes the sequence of $k$-tuples chosen by the adversary before the start of
plays. As before, $r_t:[0,1]^d \rightarrow [0,1]$ with $g_t:[0,1]^k \rightarrow
[0,1]$ where $g_t \in \calC(\alpha,\delta,L,k)$. We consider the sequence of
$k$-tuples to be at most $S$-hard implying that
$H[(i_{1,t},\dots,i_{k,t})_{t=1}^{n}] \leq S$ for some $S > 0$, and also assume
that $S$ is known to the player. We now proceed to show how a slight
modification of Algorithm CAB($d,k$) can be used to derive a bound on the regret
in this setting. Consider the optimal strategy $\vecx^{*}$ where
\begin{equation*}
\vecx^{*} \in \text{argmax}_{\vecx \in [0,1]^d} \sum_{t=1}^{n}
g_t(x_{i_{1,t}},\dots,x_{i_{k,t}}).
\end{equation*}
Since the sequence of $k$-tuples is $S$-hard, this in turn implies for any
$\vecx^{*}$ that $H[(x^{*}_{i_{1,t}},\dots,x^{*}_{i_{k,t}})_{t=1}^{n}] \leq S$.
Therefore we can now consider this as a setting where the players regret is
measured against a $S$-hard sequence
$(x^{*}_{i_{1,t}},\dots,x^{*}_{i_{k,t}})_{t=1}^{n}$.
 
Now the player does not know which $k$-tuple is chosen at each time $t$. Hence
we again construct the discrete strategy set $\calP_M$ (as defined in
\eqref{eq:strat_set_gen_k}) using the family of partitions $\calA$ of
$\set{1,\dots,d}$. By construction, we will have for any $\vecx \in [0,1]^d$ and
any $k$-tuple  $(i_1,\dots,i_k)$, the existence of a point $\vecz$ in $\calP_M$
such that $(z_{i_1},\dots,z_{i_k})$ approximates $(x_{i_1},\dots,x_{i_k})$
arbitrarily well for increasing values of $M$. Hence, for the optimal sequence
$(x^{*}_{i_{1,t}},\dots,x^{*}_{i_{k,t}})_{t=1}^{n}$, we have the existence of a
sequence of points $(\vecz^{(t)})_{t=1}^{n}$ where $\vecz^{(t)} \in \calP_M$
with the following two properties.
\begin{enumerate}
\item \textit{S-hardness.}
$H[(z^{(t)}_{i_{1,t}},\dots,z^{(t)}_{i_{k,t}})_{t=1}^{n}] \leq S$. This follows
easily from the $S$-hardness of the sequence
$(x^{*}_{i_{1,t}},\dots,x^{*}_{i_{k,t}})_{t=1}^{n}$ and by choosing for each
$(x^{*}_{i_{1,t}},\dots,x^{*}_{i_{k,t}})$ a corresponding $\vecz^{(t)} \in
\calP_M$ such that $\norm{(x^{*}_{i_{1,t}},\dots,x^{*}_{i_{k,t}}) -
(z^{(t)}_{i_{1,t}},\dots,z^{(t)}_{i_{k,t}})}$ is minimized.

\item \textit{Approximation property.}
$\norm{(x^{*}_{i_{1,t}},\dots,x^{*}_{i_{k,t}}) -
(z^{(t)}_{i_{1,t}},\dots,z^{(t)}_{i_{k,t}})}_2 = O(k^{\alpha / 2} M^{-\alpha})$.
This is easily verifiable via the projection property of the set $\calP_M$.
\end{enumerate}
Therefore by employing the Exp3.S algorithm \cite{Auer03} on the strategy set
$\calP_M$ we reduce the problem to a finite armed adversarial bandit problem
where the players regret measured against the $S$-hard sequence
$(z^{(t)}_{i_{1,t}},\dots,z^{(t)}_{i_{k,t}})_{t=1}^{n}$ is bounded from above.
The approximation property of this sequence (as explained above) coupled with
the H\"{o}lder continuity of $g_t$ ensures in turn that the players regret
against the original sequence
$(x^{*}_{i_{1,t}},\dots,x^{*}_{i_{k,t}})_{t=1}^{n}$ is also bounded. With this
in mind we present the following lemma, which formally states a bound on regret
after $T$ rounds of play.   
\begin{lemma} \label{lemma:adver_gen_k_up_bd_change}
Given the above setting and assuming that:
\begin{enumerate}
\item the sequence of $k$-tuples $(i_{1,t},\dots,i_{k,t})_{t=1}^{n}$ is at most
$S$-hard and,

\item the Exp3.S algorithm is used along with the strategy set $\calP_M$,
\end{enumerate}
we have for 
\begin{equation*}
M = \left\lceil\left(k^{\frac{\alpha-3}{2}} e^{-\frac{k}{2}} (S\log
d)^{-\frac{1}{2}} \sqrt{\frac{T}{\log
T}}\right)^{\frac{2}{2\alpha+k}}\right\rceil,
\end{equation*}
that the regret incurred by the player after $T$ rounds is given by:
\begin{equation*}
R(T) = O\left(T^{\frac{\alpha+k}{2\alpha+k}} (\log T)^{\frac{\alpha}{2\alpha+k}}
k^{\frac{\alpha(k+6)}{2(2\alpha+k)}} (S\log d)^{\frac{\alpha}{2\alpha+k}} \right).
\end{equation*}
\end{lemma}
\begin{proof}
At each time $t$, for some $\vecz^{(t)} \in \calP_M$ we can split $R(T)$ into
$R_1(T)+R_2(T)$ where
\begin{align*}
R_1(T) = \expec[\sum_{t=1}^{T} g_t(x^{*}_{i_{1,t}},\dots,x^{*}_{i_{k,t}}) -
g_t(z^{(t)}_{i_{1,t}},\dots,z^{(t)}_{i_{k,t}})], \\
R_2(T) = \expec[\sum_{t=1}^{T} g_t(z^{(t)}_{i_{1,t}},\dots,z^{(t)}_{i_{k,t}}) -
g_t(x^{(t)}_{i_{1,t}},\dots,x^{(t)}_{i_{k,t}})].
\end{align*}
Let us consider $R_1(T)$ first. As before, from the projection property of
$\calA$ we have for each $(x^{*}_{i_{1,t}},\dots,x^{*}_{i_{k,t}})$, that there
exists $\vecz^{(t)} \in \calP_M$ with 
\begin{equation*}
z^{(t)}_{i_{1,t}} = \frac{\alpha_1^{(t)}}{M},\dots,z^{(t)}_{i_{k,t}} =
\frac{\alpha_k^{(t)}}{M}
\end{equation*}
where $\alpha_1^{(t)},\dots,\alpha_k^{(t)}$ are such that $\abs{\alpha_j^{(t)}/M
- x^{*}_{i_{j,t}}} < (1/M)$ holds for $j=1,\dots,k$ and each $t=1,\dots,n$.
Therefore from H\"{o}lder continuity of $g_t$ we obtain $R_1(T) = O(T
k^{\alpha/2} M^{-\alpha})$. It remains to bound $R_2(T)$. To this end, note that
the sequence $(z^{(t)}_{i_{1,t}},\dots,z^{(t)}_{i_{k,t}})_{t=1}^{n}$ with
$\vecz^{(t)} \in \calP_M$ is at most $S$-hard. Hence the problem has reduced to
a $\abs{\calP_M}$ armed adversarial bandit problem with a $S$-hard optimal
sequence of plays against which the regret of the player is to be bounded. This
is accomplished by using the Exp3.S algorithm of \cite{Auer03} which is designed
to control regret against \textit{any} $S$-hard sequence of plays. In particular
from Corollary 8.3 of \cite{Auer03} we have that $R_2(T) = O(\sqrt{S
\abs{\calP_M}T \log(\abs{\calP_M} T)})$. Combining the bounds for $R_1(T)$ and
$R_2(T)$ and recalling that $\abs{\calP_M} = O(M^k k e^k \log d)$ we obtain the
following expression for $R(T)$:
\begin{equation}
R(T) = O(Tk^{\alpha / 2} M^{-\alpha} + \sqrt{S T M^{k} k e^{k} \log d \log(T M^k
k e^k \log d)}). \label{eq:adver_regret_changing_k_tup}
\end{equation}
Lastly after plugging in the value $M = \left\lceil\left(k^{\frac{\alpha-3}{2}}
e^{-\frac{k}{2}} (S\log d)^{-\frac{1}{2}} \sqrt{\frac{T}{\log
T}}\right)^{\frac{2}{2\alpha+k}}\right\rceil$ in
\eqref{eq:adver_regret_changing_k_tup}, it is verifiable that we obtain the
stated bound on $R(T)$.
\end{proof}
Then using strategy set $\calP_M$ with $M$ defined as above, we can employ
Algorithm \ref{alg:cont_arm_band_gen_k} with \textbf{MAB} sub-routine being the
Exp3.S algorithm. Since the regret in the inner loop over $T$ rounds is given by
$R(T)$, we have that the overall regret over $n$ plays is given by
\begin{equation*}
R(n) = \sum_{i=0,T=2^i}^{i=\log n} R(T) = O\left(n^{\frac{\alpha+k}{2\alpha+k}}
(\log n)^{\frac{\alpha}{2\alpha+k}} k^{\frac{\alpha(k+6)}{2(2\alpha+k)}}
(S\log d)^{\frac{\alpha}{2\alpha+k}}\right). 
\end{equation*}
\begin{remark}
In case the player does not know $S$, then a regret of 
\begin{equation*}
R(n) = O\left(n^{\frac{\alpha+k}{2\alpha+k}} (\log n)^{\frac{\alpha}{2\alpha+k}}
k^{\frac{\alpha(k+6)}{2(2\alpha+k)}} (\log d)^{\frac{\alpha}{2\alpha+k}} H[(\veci_t)_{t=1}^{n}]\right)
\end{equation*}
would be incurred by Algorithm \ref{alg:cont_arm_band_gen_k} with the
\textbf{MAB} routine being the Exp3.S algorithm and for the choice
\begin{equation*}
M = \left\lceil\left(k^{\frac{\alpha-3}{2}} e^{-\frac{k}{2}} (\log
d)^{-\frac{1}{2}} \sqrt{\frac{T}{\log
T}}\right)^{\frac{2}{2\alpha+k}}\right\rceil.
\end{equation*}
Here $\veci_t$ is shorthand notation for $(i_{1,t},\dots,i_{k,t})$. This can be
verified easily along the lines of the proof of Lemma
\ref{lemma:adver_gen_k_up_bd_change} by noting that on account of Corollary 8.2
of \cite{Auer03}, we have $R_2(T) = O(H[(\veci_t)_{t=1}^{n}]\sqrt{\abs{\calP_M}T
\log(\abs{\calP_M} T)})$.
\end{remark}

%% file: bandits_conclusion.tex
\section{Concluding Remarks} \label{sec:conclusion}
In this work we considered continuum armed bandit problems for the stochastic
and adversarial settings where the reward function $r:[0,1]^d \rightarrow \matR$
is assumed to depend on $k$ out of the $d$ coordinate variables implying
$r(x_1,\dots,x_d) = g(x_{i_1},\dots,x_{i_k})$ at each round. Assuming $g$ to
satisfy a local H\"{o}lder continuity condition and the $k$ tuple
$(i_1,\dots,i_k)$ to be fixed but unknown across time, we proposed a simple
modification of the CAB1 algorithm namely the $\CABdk$ algorithm which
is based on a probabilistic construction of the set of sampling points. We
proved that our algorithm achieves a regret bound that scales
sub-logarithmically with dimension $d$, with the \textit{rate} of growth of
regret depending only on $k$ and the exponent $\alpha \in (0,1]$. We then showed
how $\CABdk$ can be adapted to the setting where the $k$-tuple changes across
time, and derived a bound on its regret in terms of the hardness of the sequence
of $k$-tuples.

\textbf{Improved regret bounds.} The regret bounds derived in this paper for the case when $(i_1,\dots,i_k)$ is fixed, can be
sharpened by employing optimal finite armed bandit algorithms. In particular, 
for the adversarial setting we can use the INF algorithm of \cite{Audibert10} as the \textbf{MAB}
routine in our algorithm and get rid of the $\log n$ factor from the regret bound. For the stochastic setting,  
if the range of the reward functions was restricted to be $[0,1]$ then one can again simply employ the INF algorithm 
to get rid of the $\log n$ factor from the regret bounds. When the range of the reward functions is $\matR$, as is the case in our setting, it seems possible 
to consider a variant of the MOSS algorithm \cite{Audibert10} along with Assumption \ref{assump:dist_reward_fns} 
on the distribution of the reward functions (using proof techniques similar to \cite{KleinbergPhd}), to remove the $\log n$ factor from the regret bound. 

\textbf{Future work.} There are several interesting lines of future work. Firstly for the case when
$(i_1,\dots,i_k)$ is fixed across time it would be interesting to investigate
whether the dependence of regret on $k$ and dimension $d$ achieved by our
algorithm, is optimal or not. Secondly, for the case when $(i_1,\dots,i_k)$ can
also change with time, it would be interesting to derive lower bounds on regret
to know the optimal dependence on the hardness of the sequence of $k$
tuples. \newline

\noindent \textbf{Acknowledgments.} The authors thank Sebastian Stich for the
helpful discussions and comments on the manuscript and Fabrizio Grandoni for making us aware of 
deterministic constructions of hash functions. The project CG Learning
acknowledges the financial support of the Future and Emerging Technologies (FET)
programme within the Seventh Framework Programme for Research of the European
Commission, under FET-Open grant number: 255827.

%% file: bandits_appendix.tex
%
\appendix
\section{Deterministic construction of perfect hash functions}
\label{sec:det_constr_hash_functions}
In this section we outline for the convenience of the reader, a deterministic
construction of a family of perfect hash functions as was shown in
\cite{Naor95}. We provide only the main idea here, for details the reader is refered to \cite{Naor95}.
We proceed with the following definition.
\begin{definition}
A $(d,k,l)$ splitter $\calH$ is a family of functions $h:[d] \rightarrow [l]$ such that
for all $S \in {[d] \choose k}$ there exists $h \in \calH$ that splits $S$ perfectly i.e. into equal sized parts (or as equal as possible)
$(h^{-1}(j)) \cap S$, $j=1,\dots,l$. So if $l < k$ and $l$ divides $k$ then each $l$ is mapped to by exactly $k/l$ elements from $S$. If $l \geq k$ then it means that 
there exists an $h \in \calH$ which is injective on $S$.
\end{definition}
It is clear that a family of perfect hash functions is a $(d,k,k)$ splitter. What we aim to prove is the following theorem from \cite{Naor95}.
\begin{theorem}[Theorem 3 (iii) of \cite{Naor95}] \label{thm:app_main_thm}
There exists a deterministic construction of a $(d,k,k)$ splitter $\calH$ with $\abs{\calH} = e^k k^{O(\log k)} \log d$ in time $poly(d,k)$.
\end{theorem}
In order to prove the above theorem we would need to derive some intermediate results. To start off we have the following crucial result (Theorem 2 of \cite{Naor95}). 
\begin{theorem} \label{thm:greedy_alg_det_cons}
An $(d,k,k)$ family of perfect hash functions of size $O(e^k \sqrt{k} \log d)$ can be constructed deterministically in time $O(k^{k+1} {d \choose k} d^k/{k!})$.
\end{theorem}
\begin{proof}
Let $\calF_{d,k}$ be a $k$-wise independent probability space of random variables $x_{1},\dots,x_{d}$ such that for any $\set{i_1,\dots,i_k} \subseteq [d]$
the random variables $x_{i_1},\dots,x_{i_k}$ are \textit{mutually independent}. Assuming $x_j$ is uniformly distributed over an alphabet $A$, explicit constructions 
exist for such spaces of size $O(d^{\frac{\abs{A}-1}{\abs{A}}k})$  assuming $\abs{A}$ is prime and $d+1$ is a power of $\abs{A}$ 
(see for example - N. Alon and J.H. Spencer, The probabilistic method, John Wiley and sons Inc., New York, 1991).
An important property of these constructions is that it is possible to list all members of the probability 
space in \textit{linear} time.

Given the above let us assume that we have such a $k$-wise independent probability space with each random variable taking values in $[k]$;
hence $\abs{\calF_{d,k}} \leq d^k$. Denoting $C$ as the collection of permutations of $[k]$, let for each $S \in {d \choose k}$, $T_S$ denote all 
$h \in \calF_{d,k}$ that satisfy $C$ at $S$, i.e. $h_{|S} \in C$. Here $h_{|S}$ denotes the projection of $h$ to $S$. We now have the following important
claim.
\begin{claim}
There exists $h \in \calF_{d,k}$ that lies in at least $k!/k^k$ fraction of the sets $T_S$. 
\end{claim}
\begin{proof}
Consider $h \in \calF_{d,k}$ to be sampled uniformly at random. Then for fixed $S \in {d \choose k}$ we have that $Pr(h \in T_S) = k!/k^k$. 
In other words, $\expec_h[\mathbbm{1}_{h \in T_S}] = k!/k^k$. We then have by linearity of expectation that
$\expec_h[\sum_{S} \mathbbm{1}_{h \in T_S}] = {d \choose k} k!/k^k$ which in turn implies the claim.
\end{proof}
The claim is crucial because it allows us to simply apply an exhaustive search over all $h \in \calF_{d,k}$ and check for each $h$, the number of
sets $T_S$ in which it lies in. After finding such a good $h$ we add it to our collection, remove all the corresponding $T_S$ in which this $h$ lies and repeat.
Note that the worst case time taken to find $h$ in the $i^{th}$ iteration is ${d \choose k} \abs{\calF_{d,k}} T (1-k!/k^k)^i$ for $i=0,1,\dots$ where $T = O(k)$ is the time taken to test 
whether $h$ lies in $T_S$ or not. It then follows that the total time taken to construct the $(d,k,k)$ splitter is 
\begin{equation*}
O\left({d \choose k} \abs{\calF_{d,k}} T \sum_{i=0}^{\infty} (1-\frac{k!}{k^k})^i\right) = O\left({d \choose k} \abs{\calF_{d,k}} T \frac{k^k}{k!}\right) 
= O(k^{k+1} {d \choose k} d^k/{k!}).
\end{equation*}
We also have that the size of the family is given by the smallest $m$ such that ${d \choose k} (1-k!/k^k)^m \leq 1$ which concludes the proof.
\end{proof}

We see that the size of the family is small however the time complexity is too high. However the trick is to invoke the above theorem for ``small'' values of 
$d$ and $k$ - this will keep the overall time complexity low while leading to families of reasonable size.

 \textbf{Reducing size from $[d]$ to $[k^2]$ (Lemma $2$ of \cite{Naor95}).} Firstly, we first reduce the size of the universe by constructing a $(d,k,k^2)$ splitter $A(d,k,k^2)$ 
of size $O(k^6 \log d \log k)$. The construction essentially involves constructing an asymptotically good error correcting code with $n$ codewords over the alphabet $[k^2]$ with a
minimum relative distance of at least $1-2/k^2$ between the codewords. 
Such codes of length $L = O(k^6 \log d \log k)$ exist (see Alon et al., Construction of asymptotically good low-rate error correcting codes through pseudo random graphs,
IEEE trans. on Information Theory, 38 (1992),509-516). It is then not too hard to verify that the index set $[L]$ correspnds to a $(d,k,k^2)$ splitter. 
Furthermore the time complexity of this construction is $poly(d,k)$. 

 \textbf{Constructing $(k^2,k,\ell^{\prime})$ splitter for $\ell^{\prime} = O(\log k)$ (Lemma $3$ of \cite{Naor95}).} Secondly, we construct a $(k^2,k,\ell^{\prime})$ 
splitter $B(k^2,k,\ell^{\prime})$ for $\ell^{\prime} = O(\log k)$ (chosen to minimize the running time of the construction) using the trivial method of 
exhaustive enumeration of all $k^2 \choose \ell^{\prime}$ ``intervals''. This leads to the family $B(k^2,k,\ell^{\prime})$ of size ${k^2 \choose \ell^{\prime}}  = k^{O(\log k)}$. 

 \textbf{Constructing families of $(k^2,k/\ell^{\prime})$ perfect hash functions.} Lastly, we construct $\calP_1,\dots,\calP_{\ell{\prime}}$ where 
$\calP_1$ is a $(k^2,\lceil k/\ell^{\prime} \rceil)$ family of perfect hash functions and $\calP_2,\dots,\calP_{\ell^{\prime}}$ are $(k^2,\lfloor k/\ell^{\prime} \rfloor)$ families 
of perfect hash functions. These families are constructed using the greedy algorithm of Theorem \ref{thm:greedy_alg_det_cons} described earlier. 

Given the above families we define for every $f_1 \in A(n,k,k^2)$, $f_2 \in B(k^2,k,\ell^{\prime})$, $g_i \in \calP_i$ for all $s \in [d]$ the following function 

\begin{equation*}
 h_{f_1,f_2,g_1,g_2,\dots,g_{\ell^{\prime}}}(s) = \left\{
\begin{array}{rl}
\lceil k/\ell^{\prime} \rceil(f_2(f_1(s)) - 1) + g_{f_2(f_1(s))}(f_1(s)); & \text{if} \quad f_2(f_1(s)) = 1 \\
\lceil k/\ell^{\prime} \rceil + \lfloor k/\ell^{\prime} \rfloor(f_2(f_1(s)) - 2) + g_{f_2(f_1(s))}(f_1(s)); & \text{otherwise}.
\end{array} \right .
\end{equation*}

We define our family of $(d,k)$ perfect hash functions $\calH$ to consist of all such functions $h_{f_1,f_2,g_1,g_2,\dots,g_{\ell^{\prime}}}$. It is not too hard to
verify that $\calH$ is $(d,k)$ perfect (see proof of Theorem 3 (iii) of \cite{Naor95}). Also, 

\begin{equation*}
 \abs{H} = \abs{A(d,k,k^2)} \abs{B(k^2,k,\ell^{\prime})} \prod_{i=1}^{\ell^{\prime}} \abs{\calP_i} = e^{k} k^{O(\log k)} \log d.
\end{equation*}

Lastly each $\calP_i$ can be constructed in $2^{O(k)}$ time as per Theorem \ref{thm:greedy_alg_det_cons}. Construction of $A(d,k,k^2)$ and $B(k^2,k,\ell^{\prime})$ can be
done in $poly(2^k,d)$ time. This concludes the proof of Theorem \ref{thm:app_main_thm}.